%% file: root.tex
\documentclass[letterpaper, 10 pt, conference]{ieeeconf}  

\IEEEoverridecommandlockouts                              

\usepackage{placeins}
\usepackage{graphics} 
\usepackage{epsfig} 
\usepackage{mathptmx} 
\usepackage{times} 
\usepackage{amsmath} 
\usepackage{amssymb}  
\usepackage{algorithm,algpseudocode}
\usepackage{graphicx}
\usepackage{subcaption}
\usepackage{adjustbox}
\usepackage{color}
\usepackage{url}
\usepackage{tikz}
\usepackage{graphicx}
\usepackage{svg}
\usepackage{booktabs}
\usepackage{mathtools}
\usepackage{hyperref}

\usetikzlibrary{shapes.geometric, arrows.meta, positioning, calc, shadows, decorations.pathreplacing, fit}

\input{my_commands}

\graphicspath{{figures/}}
\title{\LARGE \bf
Learning-Based Robust Control: Unifying Exploration and Distributional Robustness for Reliable Robotics via Free Energy
}

\author{Hozefa Jesawada$^{*1}$, 
        Giovanni Russo$^{2}$, 
        Abdalla Swikir$^{3}$, 
        and Fares Abu-Dakka$^{1}$
\thanks{$^{*}$ Corresponding: {\tt\footnotesize hj3000@nyu.edu}. \textsuperscript{1}Mechanical Engineering Program, New York University Abu Dhabi, Abu Dhabi, UAE. \textsuperscript{2}Department of Information and Electrical Engineering \& Applied Mathematics, University of
Salerno, Fisciano, Italy. \textsuperscript{3}Mohammed bin Zayed University for Artificial Intelligence, Abu Dhabi, UAE. Work not supported by any organization}%
}

\begin{document}

\maketitle
\thispagestyle{empty}
\pagestyle{empty}

\begin{abstract}
A key challenge towards reliable robotic control is devising computational models that can both learn policies and guarantee robustness when deployed in the field. Inspired by the free energy principle in computational neuroscience, to address these challenges, we propose a model for policy computation that jointly learns environment dynamics and rewards, while ensuring robustness to epistemic uncertainties. Expounding a distributionally robust free energy principle, we propose a modification to the maximum diffusion learning framework. After explicitly characterizing robustness of our policies to epistemic uncertainties in both environment and reward, we validate their effectiveness on continuous-control benchmarks, via both simulations and real-world experiments involving manipulation with a Franka Research~3 arm. Across simulation and zero-shot deployment, our approach narrows the sim-to-real gap, and enables repeatable tabletop manipulation without task-specific fine-tuning.
\end{abstract}

\section{Introduction}
A popular paradigm to design autonomous robots has become that of learning control policies from simulations \cite{kober2013reinforcement}. Yet, even with high-fidelity simulators, these policies can fail when deployed in real-world conditions that deviate, even slightly, from those seen during training \cite{Kejriwal2024}. Such mismatches can arise in, e.g., contact-rich tasks, or in the presence of sensing/actuation noise, or even when some dynamics -- such as nonlinear frictions, delays, calibration drifts \cite{thrun2002probabilistic} -- is not perfectly modeled. In these settings, small discrepancies can cause large failures, with potentially catastrophic consequences for the robot and its environment. 

In this context, the minimization of the free energy -- and its instantiation into active inference robotics architectures -- has emerged as a promising approach bridging machine learning, robotics and neuroscience to design inference, learning and control algorithms for effective robot control \cite{friston2009free,Prescott2023,Vijayaraghavan2025,lanillos2021activeinferenceroboticsartificial, Baioumy2021,10160593,pezzato2020novel}. Yet, despite the success of free energy based policies, there is currently no computational model that, simultaneously, learns the policy and provides explicit robustness guarantees against epistemic uncertainties in the environment model learned by the agent and the task reward. Motivated by this, we present a computational model that: (i) without having available a model of the environment and reward function, learns a policy that minimizes the (variational) free energy; (ii) provides a-priori explicit robustness guarantees of the policy that can be used as {\em certificates} for robot deployment in the field. 

\begin{figure}[htbp]
    \centering
    \includegraphics[width=\columnwidth]{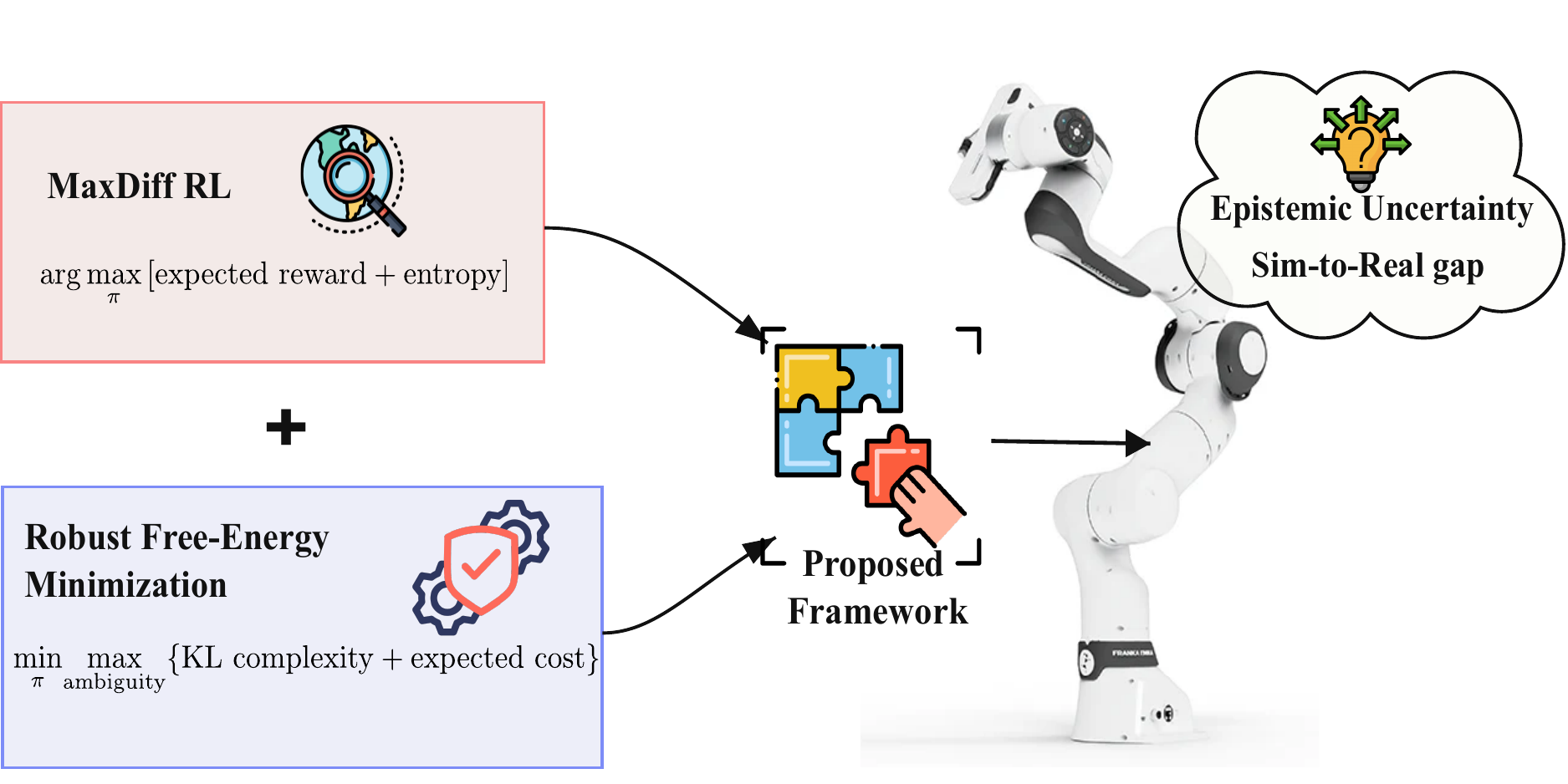}
    \caption{Modifying Maximum Diffusion RL with Distributionally Robust Free Energy Principle to address epistemic uncertainty, narrowing sim-to-real gaps for reliable robotics.}
    \label{fig:graphical-abstract}\vspace{-0.5cm}
\end{figure}

To achieve this, expounding a model-based distributional free energy principle (DR-FREE) for policy computation  \cite{shafiei2025distributionally}, we propose a modification to a recent learning framework \cite{berrueta2024maximum}, Maximum Diffusion (MaxDiff). MaxDiff, which generalizes maximum entropy methods -- see, e.g., \cite{eysenbach2021maximum,hazan2019provably,haarnoja2018soft} --  and exceeds state-of-the-art performance on popular robotics benchmarks, learns policies with maximally diffusive path statistics without having access to the environment model and reward function. However, despite its impressive control performance, MaxDiff robustness is only implicit: it depends on the entropy of the optimal policy, not known a priori \cite{eysenbach2021maximum}, and explicit bounds are only available for discrete action spaces. In contrast, DR-FREE provides explicit distributional robustness guarantees in free energy minimization problems, but requires access to both the dynamics model and the reward. Remarkably, as we recall in Section \ref{sec:background}, the MaxDiff optimization can be recast as a free energy minimization problem, enabling us to integrate DR-FREE into MaxDiff and obtain explicit robustness guarantees without sacrificing MaxDiff’s desirable features. The effectiveness of our method is illustrated via continuous control benchmarks. Numerical and hardware experiments -- involving manipulation with a Franka Emika Panda arm -- show that our method outperforms baseline MaxDiff \cite{berrueta2024maximum}, narrows the sim-to-real gap and enables zero-shot deployment of the policy.




\paragraph*{{\bf Related Work}}  the minimization of the free energy offers a unifying account across machine learning, robotics and neuroscience. See \cite{Prescott2023,Vijayaraghavan2025,lanillos2021activeinferenceroboticsartificial, Baioumy2021,10160593,pezzato2020novel} for surveys and results instantiating this framework into robots. Here, we focus on robustness in learning based control for robotics.

Adversarial and ensemble-based training strategies like RARL \cite{pinto2017robust} and EPOpt \cite{Rajeswaran2016EPOpt} expose policies to disturbances or challenging domains during training, enhancing tolerance to parametric spread. Domain randomization has been a key sim2real tactic in robotics, i.e., randomizing inertial and contact parameters for transfer \cite{peng2018sim,tobin2017domain,andrychowicz2020learning}. While effective in broadening training coverage, these methods typically optimize a \emph{nominal} objective at test time; robustness relies on the alignment between training distributions and deployment shifts, with no explicit control over deployment-time worst-case model misspecification. In contrast, we optimize a \emph{distributionally robust} objective at execution through per–state-action KL ambiguity, directly aligning decisions with epistemic risk. 

Robust MDPs and their distributionally robust variants optimize the worst-case value over uncertainty sets on transition kernels, yielding robust Bellman operators and bounds under bounded misspecification \cite{Iyengar2005,nilim2005robust,derman2020distributionalregularization}. Wasserstein and $f$-divergence sets (incl.\ KL) offer tractable inner maximizations and interpretable radii \cite{jesawada2025dr}. Our formulation instantiates \emph{per–state-action} KL balls shaped by the learned model’s uncertainty and solves the inner step via a free-energy (Donsker–Varadhan) reduction \cite{shafiei2025distributionally} to a scalar convex problem, enabling real-time planning. In contrast to most DR-MDP applications that target tabular or abstract benchmarks, we tailor ambiguity calibration and solver design to continuous-state/action robotics and couple it with a diffusion prior.

Max-entropy RL and control-as-inference (e.g., LMDP, PI$^2$, soft-actor methods) temper policies toward high entropy, offering stochastic smoothing and connections to risk sensitivity \cite{todorov_linearly_mdps}. These objectives regularize the policy distribution (or passive dynamics) \cite{gagliardi2020probabilistic}, but do not defend against adversarial shifts in the transition model. We borrow the same free-energy machinery but place the KL on next-state distributions (epistemic dynamics), yielding a robust soft policy whose conservatism adapts to model uncertainty.

Information-theoretic MPC/MPPI performs trajectory sampling with exponential cost weighting and has seen extensive robotic deployments (autonomous driving, quadrotors, manipulators), including GPU-accelerated real-time variants \cite{williams2017model,williams2018information}. Standard NN-MPPI, however, assumes a single predictive model plus temperature tuning; it does not enforce distributional robustness to transition misspecification. Our controller preserves MPPI’s practicality but embeds a principled inner robustification via per-step KL ambiguity.

Probabilistic ensembles (PETS) \cite{chua2018deep} and model-based policy optimization (MBPO) \cite{janner2019trust} achieve strong sample efficiency and competitive control on continuous-control and robotics benchmarks. Model-Ensemble Exploration and
Exploitation (MEEE) \cite{yao2021sample} tackles the problem of exploration and exploitation simultaneously by taking into account the future observation novelty and leveraging a model-ensemble. However, none of these methods provide any formal guarantees of robustness against adversarial uncertainty.

Maximum Diffusion RL \cite{berrueta2024maximum} enhances state-space coverage by amplifying diffusion and mitigating temporal correlations. MaxDiff outperforms MaxEnt RL \cite{eysenbach2021maximum}, NN-MPPI \cite{williams2018information}, and Soft-Actor Critic (SAC) \cite{haarnoja2018soft}. We incorporate a MaxDiff-style diffusion prior into our planner to ensure effective exploration in simulation, while KL ambiguity maintains robustness during execution, ensuring clear separation of concerns.

Robotics-focused RL demonstrations—SAC on real platforms, vision-in-the-loop manipulation, and sim2real locomotion \cite{levine2018learning,kumar2022adapting}—either rely on large real-world data, on-policy adaptation, or heavy randomization. In contrast, using a simulation-trained dynamics model, (i) we align control with epistemic risk via per-state KL ambiguity, (ii) preserve sample-based planning (MPPI/PETS), and (iii) couple exploration (MaxDiff) with free-energy DR robustness without deployment fine-tuning. We also show that KL-bounded dynamics robustness extends to a class of reward perturbations through an augmented-state free-energy bound.

\paragraph*{{\bf Contributions}}
we propose a modification to MaxDiff, combining it with the distributional free energy principle, DR-FREE. For our resulting computational model, we characterize policy robustness bounds and evaluate performance across both simulations, via OpenAI Gym \cite{brockman2016openai} and MuJoCo \cite{todorov2012mujoco}, and real hardware experiments with a Franka Emika Panda arm in a manipulation task. Experiments convincingly show that our framework outperforms standard MaxDiff baselines. Moreover, our model also proves useful for zero-shot deployments. In particular, we train a policy for a \textit{Franka Research~3} from a simulator, with the model available during training different from the real robotic arm. Then, after training, we deploy the trained policy on the real arm and show that, without any fine-tuning, the robot is able to complete a tabletop manipulation task. 

To the best of our knowledge, this work provides the first free energy computational model that, simultaneously, learns policies for continuous control tasks and provides explicit, a priori robustness guarantees that can be used to certify the use of the robot in field settings. Indeed: 1) in MaxDiff, and related literature, continuous policies can be learned without having access to the environment model and reward function. However, robustness only emerges a posteriori as a by-product of the entropy of the optimal policy; 2) In contrast, DR-FREE can guarantee robustness a priori. However, it requires a model and reward available to the agent.

Our framework combines the strengths of these approaches, enabling a priori robustness guarantees as in DR-FREE, while simultaneously learning policies without requiring access to models or rewards as in MaxDiff.

\section{Background}\label{sec:background}
After introducing the key notation and definitions, we formalize the optimization problems relevant to our work. Then, leveraging this problem, we introduce the key ingredients behind our computational model, MaxDiff and DR-FREE, motivating our approach to learning-based robust  control.
%
\subsection{Notation and Definitions}
We use calligraphic letters to denote sets e.g, $\mathcal{X}$ and boldface to denote vectors e.g, $\bv{x}$. A random variable is written as $\bv{V}$, with a particular realization denoted by $\bv{v}$. The \textit{probability density function} (pdf) of $\bv{v}$ is denoted by $p(\bv{v})$. Whenever integrals involving a pdf are considered, we implicitly assume that they exist. The expectation of a function $\mathbf{h}(\cdot)$ applied to $\bv{v}$ is defined as $\E_{p}[\mathbf{h}(\bv{v})] := \int \mathbf{h}(\bv{v}) p(\bv{v}) d\bv{v}$, where the integral is taken over the support of $p(\bv{v})$. The {conditional} pdf of $\bv{v}_1$ given $\bv{v}_2$ is written as $p(\bv{v}_1 \mid \bv{v}2)$. For countable sets, we use the notation $w{k_1:k_n}$, where $w_k$ represents a generic element, $k_1$ ($k_n$) denotes the index of the first (last) element, and $k_1:k_n$ represents the set of consecutive integers from $k_1$ through $k_n$. We denote by $\mathcal{D}$ the set of all probability density functions.

\begin{definition}
    The  KL divergence of $p(\bv{v})$ w.r.t $q(\bv{v})$ with $\mathrm{supp} p \subseteq \mathrm{supp} q$ is:
    \[\DKL{p}{q}:= \int_{v}p(\bv{v})\ln\frac{p(\bv{v})}{q(\bv{v})}d\bv{v}.\]
The KL divergence is a measure of the proximity of the pair of pdfs.  For discrete pmfs the integral can be substituted with a summation over $\bv{v}$.  
\end{definition}

Let $\mathcal{X}\subseteq \mathbb{K}^n$ and $\mathcal{U}\subseteq \mathbb{K}^p$ be the state and action spaces (with $\mathbb{K}\in\{\mathbb{R},\mathbb{Z}\}$), and let the decision horizon be $k=1,\dots,N$. At each $k$, we consider:
\begin{itemize}
  \item A \emph{trained (nominal) dynamics model} $\bar{p}_k(\bv{x}_{k}\mid \bv{x}_{k-1},\bv{u}_{k})$;
  \item A \emph{generative model} $\refplant{k}{k-1}$ for states and a \emph{generative prior} $\refpolicy{k}{k-1}$ for actions;
  \item the trajectory measures
\[
  p_{0:N} \!=\! p_0(\bv{x}_0)\!\!\prod_{k=1}^N\!\plant{k}{k-1}\,\pi_k(\bv{u}_{k}\mid \bv{x}_{k-1}),\]
  \[q_{0:N} \!=\! p_0(\bv{x}_0)\!\!\prod_{k=1}^N\! \refplant{k}{k-1}\,\refpolicy{k}{k-1}.
\]

\item stage costs $c^{(x)}_k:\mathcal{X}\!\to\!\mathbb{R}_{\ge0}$ and $c^{(u)}_k:\mathcal{U}\!\to\!\mathbb{R}_{\ge0}$.
\end{itemize}

\subsection{Free-Energy Minimization Problem}
Free-energy principle arises naturally in decision-making across 
information theory, active inference, learning, and control. In RL, 
they underpin schemes such as Maximum Entropy RL 
\cite{eysenbach2021maximum,berrueta2024maximum}, Bayesian inference 
\cite{lanillos2021activeinferenceroboticsartificial}, and 
entropy-based transport methods.

In our setting, the free-energy functional is
\begin{equation}\label{eqn:free-energy-functional}
    \sF(p_{0:N}) =
    \underbrace{\DKL{p_{0:N}}{q_{0:N}}}_{\text{complexity}} +
    \underbrace{\E_{p_{0:N}}\!\left[\sum_{k=1}^{N}
    c^{(x)}_{k}(\bv{x}_{k})+c^{(u)}_{k}(\bv{u}_{k})\right]}_{\text{expected cost}},
\end{equation}
where the first term measures divergence from a reference distribution 
$q_{0:N}$ and the second is expected cost. Notably, by 
defining a cost-shaped reference 
$\tilde{q}_{0:N} \propto q_{0:N}\exp(-\sum_k c^{(x)}_k+c^{(u)}_k)$, 
\eqref{eqn:free-energy-functional} reduces to 
$\sF(p_{0:N}) = \DKL{p_{0:N}}{\tilde{q}_{0:N}}$, 
shows that minimization of \eqref{eqn:free-energy-functional} is equivalent to KL projection onto a cost-weighted reference.


\subsection{Maximum Diffusion (MaxDiff) RL}

MaxDiff derives a path-distribution $P_{\max}$ that \emph{maximizes trajectory entropy} under continuity constraints, yielding Markov and ergodic sample paths and diffusion-like exploration. Writing the finite-horizon path entropy via the chain rule,
\[
S\!\bigl[P_{\max}[x_{1:N}]\bigr]
=\sum_{t=1}^N S\!\bigl[p_{\max}(\bv{x}_{t+1}\mid x_t)\bigr]
\;\propto\; \tfrac{1}{2}\sum_{t=1}^N \log\!\det C[x_t],
\]
where $C[x]$ is the local state-increment covariance; thus maximizing path entropy promotes large $\log\det C[x]$ along the trajectory.

A practical MaxDiff RL objective (SOC form) augments rewards with a diffusion bonus:
\begin{equation}\label{eqn:maxdiff-prob}
\begin{split}
    \pi^\star_{\mathrm{MaxDiff}}
\in\arg\max_{\pi}
\mathbb{E}_{(\bv{x}_{1:N},u_{1:N})\sim P_\pi}&\Bigl[
\sum_{k=1}^N r(\bv{x}_k,u_k)\\+\frac{\alpha}{2}\log\det C_\pi[x_k]
\Bigr],
\end{split}
\end{equation}
which is the objective used for the empirical results in \cite{berrueta2024maximum}. Equivalently, MaxDiff can be cast as minimizing
$D_{\mathrm{KL}}\!\bigl(P_\pi\Vert P^{r}_{\max}\bigr)$ between the agent’s path distribution and a maximally diffusive, reward-shaped reference $P^{r}_{\max}$, leading to an SOC with modified running reward or cost.

\subsection{Distributionally Robust Free-Energy Principle}

DR-FREE~\cite{shafiei2025distributionally} defines a sequential min–max problem over 
policies $\{\pi_k(\bv{u}_k\mid\bv{x}_{k-1})\}_{1:N}$ and environment models 
$\{p_k(\bv{x}_k\mid\bv{x}_{k-1},\bv{u}_k)\}_{1:N}$:
\begin{equation}\label{eqn:dr-free-problem}
  \min_{\{\pi_k\}}\;
  \max_{\{p_k\in\sB_\eta(\bar p_k)\}}
     \sF(p_{0:N}),
\end{equation}
subject to $p_k\!\in\!\sB_\eta(\bar p_k)$ for all $k=1,\dots,N$.
The ambiguity set is
\[
  \sB_{\eta}\!\bigl(\bar p_k\bigr)
  \;=\;
  \Bigl\{\,p_k\in\sD:\;
    D_{\mathrm{KL}}\!\bigl(p_{k}\,\Vert\,\bar{p}_{k}\bigr)
    \le \eta_k(\bv{x}_{k-1},\bv{u}_{k}),\;
  \Bigr\},
\]



where $\eta_k(\bv{x}_{k-1},\bv{u}_k)\ge 0$ bounds the statistical complexity 
relative to the trained model~$\bar{p}_k$, and the support constraint ensures 
the free-energy objective remains finite.

At each time-step~$k$, the resolution engine solves a stagewise bi-level problem:

\begin{equation}\label{eqn:stagewise}
\begin{multlined}[t]
  \min_{\pi_k}\;
  D_{\mathrm{KL}}\!\bigl(\pi_k\Vert q^{(u)}_{k}\bigr)
  + \mathbb{E}_{\pi_k}\!\Bigl[
    c^{(u)}_k(\bv{U}_k) \\
    +\!\max_{p_k\in\sB_\eta(\bar p_k)}\!
    \Bigl\{D_{\mathrm{KL}}\!\bigl(p_k\Vert q^{(x)}_{k}\bigr)
    +\mathbb{E}_{p_k}\!\bigl[\bar c_k(\bv{X}_k)\bigr]\Bigr\}
  \Bigr],
\end{multlined}
\end{equation}
where the cost-to-go is defined recursively as 
$\bar c_k(\bv{x}_k)=c^{(x)}_k(\bv{x}_k)+\hat c_{k+1}(\bv{x}_k)$, 
with $\hat c_{k+1}(\bv{x}_k)$ being the optimal value of~\eqref{eqn:stagewise} 
evaluated at $k\!+\!1$ (initialized at $\hat c_{N+1}\!=\!0$).
The inner maximization reduces to a scalar convex optimization problem 
whose non-negative optimal value yields the \emph{cost of ambiguity} 
$\eta_k(\bv{x}_{k-1},\bv{u}_k)+\tilde c(\bv{x}_{k-1},\bv{u}_k)$.

The optimal policy has an explicit Gibbs form:
\begin{equation}\label{eqn:Gibbs-policy}
\begin{split}
    &\pi_k^\star(u\mid \bv{x}_{k-1})
  \propto
  \\ &q_k(u\mid \bv{x}_{k-1})\,
  \exp\!\Bigl(-\,c^{(u)}_k(u)\;-\;\eta_k(\bv{x}_{k-1},\bv{u})\;-\;\tilde c(\bv{x}_{k-1},\bv{u})\Bigr),
\end{split}
\end{equation}
i.e., the generative prior $q_k(\bv{u}_k\mid\bv{x}_{k-1})$ is modulated by an 
exponential kernel containing the action cost and the cost of ambiguity. 
Actions associated with higher ambiguity are assigned lower probability, 
while as ambiguity vanishes, the policy recovers that of a standard 
free-energy minimizing agent.

\section{Modifying MaxDiff with DR-FREE}

First, we note that, given the reward $r=-(c^{(x)}+c^{(u)})$, one can define
\[
\tilde q_{0:N}
\;\propto\;
q_{0:N}\,\exp\!\Bigl(-\sum_{k=1}^N(c^{(x)}_k(\bv{X}_{k})+c^{(u)}_k(\bv{U}_{k}))\Bigr).
\]
Then, as noted in~\cite{shafiei2025distributionally}, when there is no ambiguity the DR-FREE objective can be recast as the minimization of the KL divergence between $p_{0:N}$ and $\tilde q_{0:N}$~\cite{garrabe2025convex}. In turn, this means that the optimization problem becomes the same problem tackled by MaxDiff when  the generative model is chosen as $\refplant{k}{k-1}=p_{\max}(\bv{x}_{k}\mid \bv{x}_{k-1})$ and a fixed action prior $\refpolicy{k}{k-1}$ (e.g., $\bar p(\bv{u}_{k}\mid \bv{x}_{k-1})$), so that the reference density becomes
$\propto \prod_k p_{\max}(\bv{x}_{k}\mid \bv{x}_{k-1}) \exp(\sum_k r(\bv{x}_{k},\bv{u}_{k}))$.



\subsection{Modifying MaxDiff: constructing $p_{\max}$ and using it as the state generative model}

\paragraph{Principle}
in DR-FREE, the complexity term $D_{\mathrm{KL}}\!\bigl(p_{0:N}\Vert q_{0:N}\bigr)$ biases the optimal policy toward a \emph{generative model} $q_{0:N}$ whose single-step factors are $q_k(\bv{x}_{k}\!\mid \bv{x}_{k-1},\bv{u}_{k})$ and $q_k(\bv{u}_{k}\!\mid \bv{x}_{k-1})$. To inject MaxDiff’s exploration into this bias, we \emph{choose the state generative kernel to be maximally diffusive}, i.e.
\[
q_k(\bv{x}_{k}\!\mid \bv{x}_{k-1},\bv{u}_{k})\;\equiv\;p_{\max}\bigl(\bv{x}_{k}\!\mid \bv{x}_{k-1},\bv{u}_{k}\bigr),
\]
while keeping any convenient prior $q_k(\bv{u}_{k}\!\mid \bv{x}_{k-1})$ for actions. This choice plugs MaxDiff’s path-entropy objective into DR-FREE via the complexity term, while the ambiguity machinery (inner maximization over $p_k$) and the soft optimal policy remain as in DR-FREE (scalar, convex inner step; Gibbs-form policy kernel).

\paragraph{What $p_{\max}$ means (MaxDiff)}
MaxDiff defines maximally diffusive Markov transitions by maximizing path entropy; for Gaussian increments this reduces to maximizing the sum of local entropies,
\[
S\!\bigl[P_{\max}[x_{1:N}]\bigr]
= \sum_{t=1}^{N} S\!\bigl[p_{\max}(\bv{x}_{t+1}\!\mid x_t)\bigr]
\;\propto\; \tfrac{1}{2}\sum_{t=1}^{N}\log\!\det C[x_t],
\]
which yields the implemented MaxDiff-RL objective
\(
\max_{\pi}\,\mathbb{E}\!\left[\sum_t r(\bv{x}_t,u_t) + \frac{\alpha}{2}\log\!\det C_\pi[x_t]\right]
\). In practice, when $C[x]$ can be low-rank/high-dimensional, one may use leading-eigenvalue or log-trace approximations for numerical stability (with the usual caveats).

\paragraph{Computing $p_{\max}$ for DR-FREE: a KL–MaxEnt construction}
for each $(\bv{x}_{k-1},\bv{u}_{k})$, the trained nominal dynamics are $\bar p_k(\cdot\!\mid \bv{x}_{k-1},\bv{u}_{k})$. We compute a \emph{single-step} maximally diffusive kernel by solving the \emph{maximum-entropy trust-region} problem
\begin{equation}
\label{eq:pmax_opt}
\begin{aligned}
p_{\max}(\cdot\!\mid \bv{x}_{k-1}, \bv{u}_{k})& \in \arg\max_{p\in\sD}\mathcal{H}\!\bigl(p\bigr)
\\&
\text{s.t.} D_{\mathrm{KL}}\!\bigl(p\Vert \bar p_k(\cdot\!\mid \bv{x}_{k-1},\bv{u}_{k})\bigr)\;\le\;\varepsilon_k(\bv{x}_{k-1},\bv{u}_{k}),\\&\quad 
\mathrm{supp}\,p\subseteq\mathrm{supp}\,\bar p_k.
\end{aligned}
\end{equation}
Lagrangian calculus gives $p_{\max}(\mathrm{d}x') \propto \bar p_k(\mathrm{d}x')^{\beta_k}$ with a unique $\beta_k\!\in\!(0,1]$ chosen to meet the KL budget; in exponential families this is equivalent to \emph{flattening} the nominal precision (increasing entropy) while staying within KL radius.

\medskip
\noindent\textbf{Gaussian closed form:} 
assume $\bar p_k(\cdot\!\mid \bv{x}_{k-1},\bv{u}_{k})=\mathcal{N}\!\bigl(\bar\mu_{k},\bar\Sigma_k\bigr)$ with $\bar \mu_{k}:=\bar f_k(\bv{x}_{k-1},\bv{u}_{k})$. Then
\[
p_{\max}(\cdot\!\mid \bv{x}_{k-1},\bv{u}_{k})
=\mathcal{N}\!\bigl(\bar\mu_{k},\;\Sigma_{\max,k}\bigr),
\,
\Sigma_{\max,k}\;=\;\lambda_k\,\bar\Sigma_k,\, \lambda_k>0,
\]
with \emph{scalar} $\lambda_k$ determined by the KL constraint. If we enforce $D_{\mathrm{KL}}\!\bigl(p_{\max}\Vert \bar p_k\bigr)=\varepsilon_k$, the unique solution satisfies
\begin{equation}\label{eq:lambda-from-eps}
\frac{n}{2}\,\bigl(\lambda_k-1-\log \lambda_k\bigr) \;=\; \varepsilon_k \quad(\lambda_k\ge 1),
\end{equation}
so $\Sigma_{\max,k}$ is a \emph{uniform inflation} of $\bar\Sigma_k$ that maximizes entropy, i.e., $\log\!\det\Sigma_{\max,k} = \log\!\det\bar\Sigma_k + n\log\lambda_k$ increases monotonically with $\lambda_k$—exactly aligned with MaxDiff’s $\frac{1}{2}\log\!\det(\cdot)$ objective. Equation \eqref{eq:lambda-from-eps} is a one-dimensional root find (bisection/Newton), strictly increasing in $\lambda_k$.


\paragraph{Plug-in to DR-FREE}
Once $p_{\max}$ is computed, set 
$q_k(\bv{x}_{k}\!\mid\bv{x}_{k-1},\bv{u}_k)\leftarrow 
 p_{\max}(\bv{x}_{k}\!\mid\bv{x}_{k-1},\bv{u}_k)$ 
and proceed with DR-FREE:
(i)~the \emph{inner} maximization of
$D_{\mathrm{KL}}(p_k\Vert q_k)
+\mathbb{E}_{p_k}[\bar c_k(\bv{X}_k)]$
over $p_k\!\in\!\sB_\eta(\bar p_k)$
yields the cost of ambiguity
$\eta(\bv{x}_{k-1},\bv{u}_k)+\tilde c(\bv{x}_{k-1},\bv{u}_k)$,
where $\tilde c(\bv{x}_{k-1},\bv{u}_k)=\min_{\alpha\ge0}\tilde V_\alpha(\bv{x}_{k-1},\bv{u}_k)$
is obtained via scalar convex optimization%
\footnote{Explicitly,
$\tilde V_\alpha
=\alpha\ln\mathbb{E}_{\bar p_k}\!
\bigl[\bigl(
  \bar p_k\exp{\bar c_k}/q_k
\bigr)^{1/\alpha}\bigr]
+\alpha\,\eta_k$
for $\alpha>0$; see~\cite{shafiei2025distributionally} 
for the boundary case $\alpha=0$.};
(ii)~the \emph{outer} minimization over $\pi_k$ returns
the Gibbs policy~\eqref{eqn:Gibbs-policy}.

\paragraph{Ambiguity-free limit}
when $\eta_k\!\to\!0$ the DR-FREE policy reduces to the free-energy minimizer that explicitly contains $D_{\mathrm{KL}}\!\bigl(\bar p_k\Vert q_k\bigr)$ and $\mathbb{E}_{\bar p_k}\![\bar c_k]$ in its exponent; with the MaxDiff plug-in $q_k=p_{\max}$ (Gaussian), this term is
\begin{equation*}
    D_{\mathrm{KL}}\!\bigl(\mathcal{N}(\bar\mu_{k},\bar\Sigma_k)\,\Vert\,\mathcal{N}(\bar\mu_{k},\lambda_k\bar\Sigma_k)\bigr)
=\frac{n}{2}\!\left(\frac{1}{\lambda_k}-1+\log\lambda_k\right),
\end{equation*}
which is a closed-form function of the $\lambda_k$ used to build $p_{\max}$.

\paragraph{When estimating through data}
when estimating a state-diffusion matrix $C[x]$ from data as in \cite{eysenbach2021maximum}, use the log-det objective or its stable approximations for large $n$. 
Any $q_k$ is admissible in DR-FREE (it can be a time-series model or encode goals); the plug-in $q_k=p_{\max}$ is therefore fully compatible with the framework.

\begin{figure}[t]
	\centering
	\def\svgwidth{0.96\linewidth}
	{\fontsize{10}{10}
		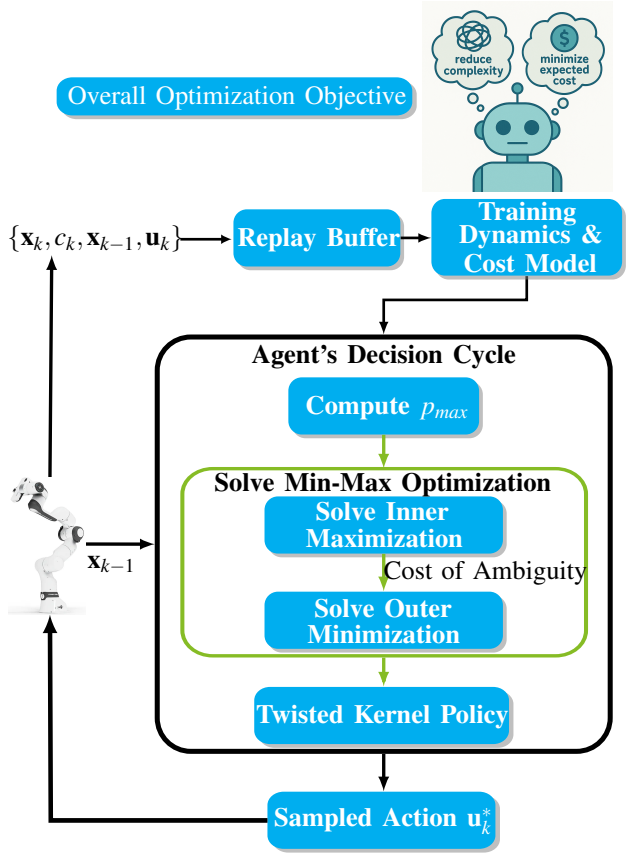}
    \caption{overview of the proposed control loop. 
At each step, the agent collects data into a replay buffer, updates 
the dynamics and cost models, computes the maximally diffusive kernel 
$p_{\max}$, and solves the min--max optimization. 
The resulting twisted kernel produces an action 
that is applied to the system, and the cycle repeats.} 
    \label{fig:control-scheme}
\end{figure}
\begin{algorithm}[!h]
\caption{Our Computational Model }
\begin{algorithmic}[1]
\Require Horizon $N$, ambiguity radius $\eta_k(\bv{x}_{k-1},\bv{u}_k)$, MaxDiff trust $\varepsilon_q$, 
action prior $q(\bv{u}_k|\bv{x}_{k-1})$.
\State Initialize the replay buffer $\mathbb{D}$, nominal dynamics model $\bar{p}(\bv{x_{k}}\mid\bv{x}_{k-1},\bv{u}_{k})$ and cost model $c_{k}(\bv{x}_{k},\bv{u}_{k})$ as neural networks.
\For{episode $e = 1,2,\dots,E$}
    \State Initialize state $\bv{x}_0$
    \For{time step $k = 0,1,\dots,N-1$}
        \State \textbf{Build $p_{max}$} using \eqref{eq:pmax_opt}--\eqref{eq:lambda-from-eps}.
        \State Sample random action from $\mathcal{U}$.
        \State \textbf{Solve} scalar convex optimization as in \eqref{eqn:c_tilde}.
        \State \textbf{Solve} outer minimization \eqref{eqn:dr-free-problem} and obtain optimal policy \eqref{eqn:Gibbs-policy}.
        \State \textbf{Execute action}: sample $\bv{u}_k \sim \pi^{*}_k(\cdot|\bv{X}_k)$, apply to system, collect $\{\bv{x}_{k}, c_k, \bv{x}_{k-1}, \bv{u}_{k}\}\rightarrow\mathbb{D}$.
    \EndFor
    \State \textbf{Update} $\bar{p}(\bv{x_{k}}\mid\bv{x}_{k-1},\bv{u}_{k})$ and $c_{k}(\bv{x}_{k},\bv{u}_{k})$ using $\mathbb{D}$ and backpropagation.
\EndFor
\end{algorithmic}
\end{algorithm}

The overall decision cycle is illustrated in Fig.~\ref{fig:control-scheme}, 
which complements Algorithm~1 by showing the interaction between replay buffer, model training, 
the computation of $p_{\max}$, the min--max optimization, and action execution.

\section{Joint Robustness to Dynamics and Cost Perturbations}\label{sec:augmented}

In this section, we extend our framework to account not only for dynamics misspecification but also for bounded perturbations in the stage cost. 
The idea of augmenting the system state with a running-cost variable, so that cost perturbations appear as additional uncertainty in the transition kernel, was briefly discussed in \cite{eysenbach2021maximum}. 
Our formal analysis, however, builds directly on the distributional free-energy results in \cite{shafiei2025distributionally}. 
Leveraging these results, we show that the policy obtained by solving \eqref{eqn:dr-free-problem} remains robust under stage cost perturbations. 
This extension allows the adversary’s KL budget to be allocated jointly across dynamics and cost channels, while preserving the tractable min--max structure and Gibbs-form policy of DR-FREE.

\subsection{Augmented State Formulation}

Let the nominal (stage) cost be
\[
\ell(\bv{x},\bv{u}) := c^{(x)}(\bv{x}) + c^{(u)}(\bv{u}).
\]
We reserve $\tilde c(\bv{x},\bv{u})$ exclusively for the \emph{ambiguity cost} returned by the
DR-FREE inner maximization. To account for perturbations
$\delta c(\bv{x},\bv{u})$ with $|\delta c|\le \Delta_k(\bv{x},\bv{u})$, we augment the observation
with the running cost $C_k$, so the augmented state is $(\bv{X}_k,C_k)$.

The corresponding transition kernels are
\begin{align*}
\bar p_k^{\mathrm{aug}}(\bv{x}',c'\mid \bv{x},C,\bv{u})
&=\bar p_k(\bv{x}'\mid \bv{x},\bv{u})\;\mathcal N\!\big(c'\,\big|\,C+\ell(\bv{x},\bv{u}),\,\sigma^2\big),
\\
p_k^{\mathrm{aug},\delta}(\bv{x}',c'\mid \bv{x},C,\bv{u})
&=p_k(\bv{x}'\mid \bv{x},\bv{u})\;\\
&\quad\mathcal N\!\big(c'\,\big|\,C+\ell(\bv{x},\bv{u})+\delta c(\bv{x},\bv{u}),\,\sigma^2\big),
\end{align*}
where $\bar p_k$ is the nominal next-state model and $p_k$ is an admissible (possibly misspecified)
next-state model. As in DR-FREE, we impose per–state-action KL balls:
\[
\DKL{p_k}{\bar p_k}\le \eta_k^{\mathrm{dyn}}(\bv{x},\bv{u}).
\]

\begin{theorem}\label{thm:joint-robustness}
Suppose the augmented ambiguity radii obey
\begin{equation}
\eta_k^{\mathrm{aug}}(\bv{x},\bv{u})\;\ge\; \eta_k^{\mathrm{dyn}}(\bv{x},\bv{u})\;+\;\frac{\Delta_k(\bv{x},\bv{u})^2}{2\sigma^2}
\quad\text{for all }(\bv{x},\bv{u}),
\label{eq:radius}
\end{equation}
Then, for each \(k\) and every \((\bv{x},\bv{u})\), the true augmented kernels
\(p^{\mathrm{aug},\delta}_k\) lie in the DR-FREE augmented ambiguity sets
\(\mathcal B_{\eta_k^{\mathrm{aug}}}(\bar p^{\mathrm{aug}}_k)\), and the saddle-point policy solving
\begin{equation}
\min_{\{\pi_k\}}\;\max_{\{p^{\mathrm{aug}}_k\in\mathcal B_{\eta_k^{\mathrm{aug}}}(\bar p^{\mathrm{aug}}_k)\}}\;
D_{\mathrm{KL}}\!\bigl(p^{\mathrm{aug}}_{0:N}\,\Vert\,\tilde q^{\mathrm{aug}}_{0:N}\bigr)
\label{eq:drfree-aug}
\end{equation}
is \emph{simultaneously robust} to both dynamics and cost perturbations, while preserving the DR-FREE
Gibbs-form policy.
\end{theorem}

\begin{proof}
\textbf{Step 1 — KL decomposition and Gaussian bound.}
By the KL chain rule on \((\bv{X}_{k+1},C_{k+1})\),
\begin{align}
&D_{\mathrm{KL}}\!\bigl(p^{\mathrm{aug},\delta}_k \,\Vert\, \bar p^{\mathrm{aug}}_k\bigr)
= D_{\mathrm{KL}}\!\bigl(p_k(\cdot\mid \bv{x},\bv{u})\,\Vert\, \bar p_k(\cdot\mid \bv{x},\bv{u})\bigr) \notag\\
&\quad + \mathbb E_{p_k(\bv{x}'\mid \bv{x},\bv{u})}\!\Bigl[
D_{\mathrm{KL}}\!\bigl(\mathcal N(\cdot\mid C+\ell+\delta c,\sigma^2)\,\Vert\,\notag\\ &\quad\mathcal N(\cdot\mid C+\ell,\sigma^2)\bigr)
\Bigr]. \label{eq:kl-split-clean}
\end{align}
For Gaussians with equal variance, the second term equals \((\delta c)^2/(2\sigma^2)\le \Delta_k(\bv{x},\bv{u})^2/(2\sigma^2)\).
Using \(D_{\mathrm{KL}}(p_k\Vert \bar p_k)\le \eta_k^{\mathrm{dyn}}(\bv{x},\bv{u})\), we obtain
\[
D_{\mathrm{KL}}\!\bigl(p^{\mathrm{aug},\delta}_k \,\Vert\, \bar p^{\mathrm{aug}}_k\bigr)\;\le\;
\eta_k^{\mathrm{dyn}}(\bv{x},\bv{u})+\frac{\Delta_k(\bv{x},\bv{u})^2}{2\sigma^2}
= \eta_k^{\mathrm{aug}}(\bv{x},\bv{u}),
\]
so \(p^{\mathrm{aug},\delta}_k\in \mathcal B_{\eta_k^{\mathrm{aug}}}(\bar p^{\mathrm{aug}}_k)\) pointwise.

\textbf{Step 2 — Pathwise feasibility.}
Per-step feasibility composes along the horizon, hence
\(p^{\mathrm{aug},\delta}_{0:N}\in \prod_{k}\mathcal B_{\eta_k^{\mathrm{aug}}}(\bar p_k^{\mathrm{aug}})\).
For any fixed \(\{\pi_k\}\),
\[
D_{\mathrm{KL}}\!\bigl(p^{\mathrm{aug},\delta}_{0:N}\,\Vert\,\tilde q^{\mathrm{aug}}_{0:N}\bigr)
\ \le\
\sup_{\{p^{\mathrm{aug}}_k\in\mathcal B_{\eta_k^{\mathrm{aug}}}\}}
D_{\mathrm{KL}}\!\bigl(p^{\mathrm{aug}}_{0:N}\,\Vert\,\tilde q^{\mathrm{aug}}_{0:N}\bigr).
\]

\textbf{Step 3 — Inner maximization, Lagrangian, and ambiguity cost.}
At a fixed \((\bv{x},\bv{u})\), introduce the multiplier \(\lambda_k(\bv{x},\bv{u})\ge 0\) for the per-step KL constraint.
Writing \(p_k^{(x)}\) for the next-state component and \(p_k^{(r)}\) for the running-cost component,
the inner Lagrangian is
\begin{align*}
&\mathcal L_k\big(p_k^{(x)},p_k^{(r)},\lambda_k\big)
=\mathbb E_{p_k^{(x)}p_k^{(r)}}\!\big[R_k+V_{k+1}(\bv{X}_{k+1})\big]\\ &- \lambda_k D_{\mathrm{KL}}(p_k^{(x)}\Vert \bar p_k^{(x)}) -\lambda_k\mathbb E_{p_k^{(x)}} D_{\mathrm{KL}}\!\big(p_k^{(r)}\Vert \bar p_k^{(r)}\big)
+ \lambda_k\eta^{\text{aug}}_k(\bv{x},\bv{u}).
\end{align*}
Where, $R_{k}:=C_{k+1}-C_{k}$ and $V_{k+1}(\bv{X}_{k+1})\equiv\hat{c}_{k+1}(\bv{X}_{k+1})$. Maximizing over \(p_k^{(r)}\) (pointwise) gives the running-cost dual term
\[
\psi_k(\bv{x},\bv{u};\lambda_k)
= \lambda_k\log\mathbb E_{\bar p_k^{(r)}}\!\left[e^{R_k/\lambda_k}\right]
= \ell(\bv{x},\bv{u})+\frac{\sigma^2}{2\lambda_k}.
\]
Then, maximizing over \(p_k^{(x)}\) yields the DV dual
\[
\widetilde V_k(\bv{x},\bv{u};\lambda_k)
= \lambda_k\log \mathbb E_{\bar p_k^{(x)}(\cdot\mid \bv{x},\bv{u})}
\!\left[\exp\!\Big(\tfrac{V_{k+1}(\bv{X}_{k+1})+\psi_k(\bv{x},\bv{u};\lambda_k)}{\lambda_k}\Big)\right].
\]
Hence, the \emph{ambiguity cost} at \((\bv{x},\bv{u})\) is the scalar convex program
\begin{equation}\label{eqn:c_tilde}
    \tilde c(\bv{x},\bv{u})
= \min_{\lambda_k\ge 0}\Big\{\widetilde V_k(\bv{x},\bv{u};\lambda_k)+\lambda_k\,\eta^{\text{aug}}_k(\bv{x},\bv{u})\Big\}.
\end{equation}

\textbf{Step 4 — Policy form.}
Substituting the inner value into the outer step (with action prior and control cost) yields the DR-FREE Gibbs policy in \eqref{eqn:Gibbs-policy}.

\textbf{Step 5 — Conclusion.}
Steps~1–4 establish that the augmented radius \eqref{eq:radius} ensures robustness to joint dynamics
and cost perturbations, with the same implementable inner–outer structure as DR-FREE.
\end{proof}

\paragraph{Discussion.}
The result strictly extends DR-FREE: when \(\Delta_k\equiv 0\) (no cost perturbations),
\(\eta_k^{\mathrm{aug}}=\eta_k^{\mathrm{dyn}}\) and the formulation reduces to the original DR-FREE engine. When
\(\Delta_k>0\), the adversary’s KL budget can be allocated across dynamics and the cost channel, yet the
ambiguity cost still arises from a \emph{one-dimensional} convex minimization over \(\lambda_k\), preserving
tractability and the Gibbs policy form.

\begin{remark}
The resolution engine still yields the exponential policy kernel
$\refpolicy{k}{k-1}\exp\!\big(-c_k^{(u)}(\bv{u}_{k})-\eta_k(\bv{x}_{k-1},\bv{u}_{k})-\tilde c(\bv{x}_{k-1},\bv{u}_{k})\big)$,
whose normalization gives $\pi_k^\star(\bv{u}_{k}\mid \bv{x}_{k-1})$. The only change is that the adversary’s KL budget can now flow into the cost channel, which is precisely the claimed robustness.
\end{remark}

\section{Experiments}
\label{sec:experiments}

We evaluate the proposed controller in three settings with two sources of epistemic uncertainty: (i-ii) \emph{model learning error}, where a neural network approximates the true dynamics but inevitably incurs prediction errors (HalfCheetah, Franka simulation), and (iii) \emph{sim-to-real transfer}, where a simulation-trained model is deployed on physical hardware (Franka Research~3). 
All experiments use identical greedy controllers (horizon $N{=}1$) for both the proposed method and baselines, isolating the effect of the robustness mechanism. The reference control prior $q_k(u_k|x_k)$ is uniform over the action space.

We set $\eta_k(\bv{x}_{k-1},\bv{u}_k) = D_{\mathrm{KL}}\big(\mathcal{N}(\bv{x}^\star, \Sigma_\star) \,\|\, \bar{p}_k(\bv{x}_{k}|\bv{x}_{k-1},\bv{u}_k)\big),$
linking robustness to model uncertainty. For sensitivity analysis, we scale by coefficient $\varrho$: $\eta_k^{\text{eff}} = \varrho \cdot \eta_k(\bv{x}_{k-1},\bv{u}_k)$. The baseline $\varrho{=}0$ recovers MaxDiff, $\varrho{=}1$ is the default, $\varrho>1$ increases conservatism. For all the experiments, ambiguity arises from mismatch between the learned neural network dynamics model $\bar{p}_k(\bv{x}_{k}\!\mid \bv{x}_{k-1},\bv{u}_{k})$ and true dynamics, and also the Gaussian noise injected into the environment reward.

Given this setup, the resulting greedy policy takes the Gibbs form
\begin{equation}
  \pi^\star(\bv{u}_{k}\!\mid \bv{x}_{k-1})\ \propto\ \exp\!\big(-\eta(\bv{x}_{k-1},\bv{u}_{k})-\tilde{c}(\bv{x}_{k-1},\bv{u}_{k})\big),
  \qquad N=1.
  \label{eq:gibbs-policy}
\end{equation}

We report (i) episode \emph{return}, and, for manipulation tasks,
(ii) the \emph{minimum distance-to-goal} with a 5\,cm success threshold.
Figures referenced below show mean \(\pm\) one standard deviation across
seeds. Across all tasks the maximum episode length is set at $1000$ steps. For implementation details and code refer to: \hyperlink{https://tinyurl.com/2v9yc9dk}{https://tinyurl.com/2v9yc9dk}.

\subsection{HalfCheetah-v5 (MuJoCo)}
\label{subsec:hc}

\textbf{Task and priors.}
The desired state \(\bv{x}^\star\) encodes target forward velocity and an
approximately upright pose. The controller uses the greedy policy
\eqref{eq:gibbs-policy} at every step. \

\begin{figure}[h!]
  \centering
    \includegraphics[width=\columnwidth]{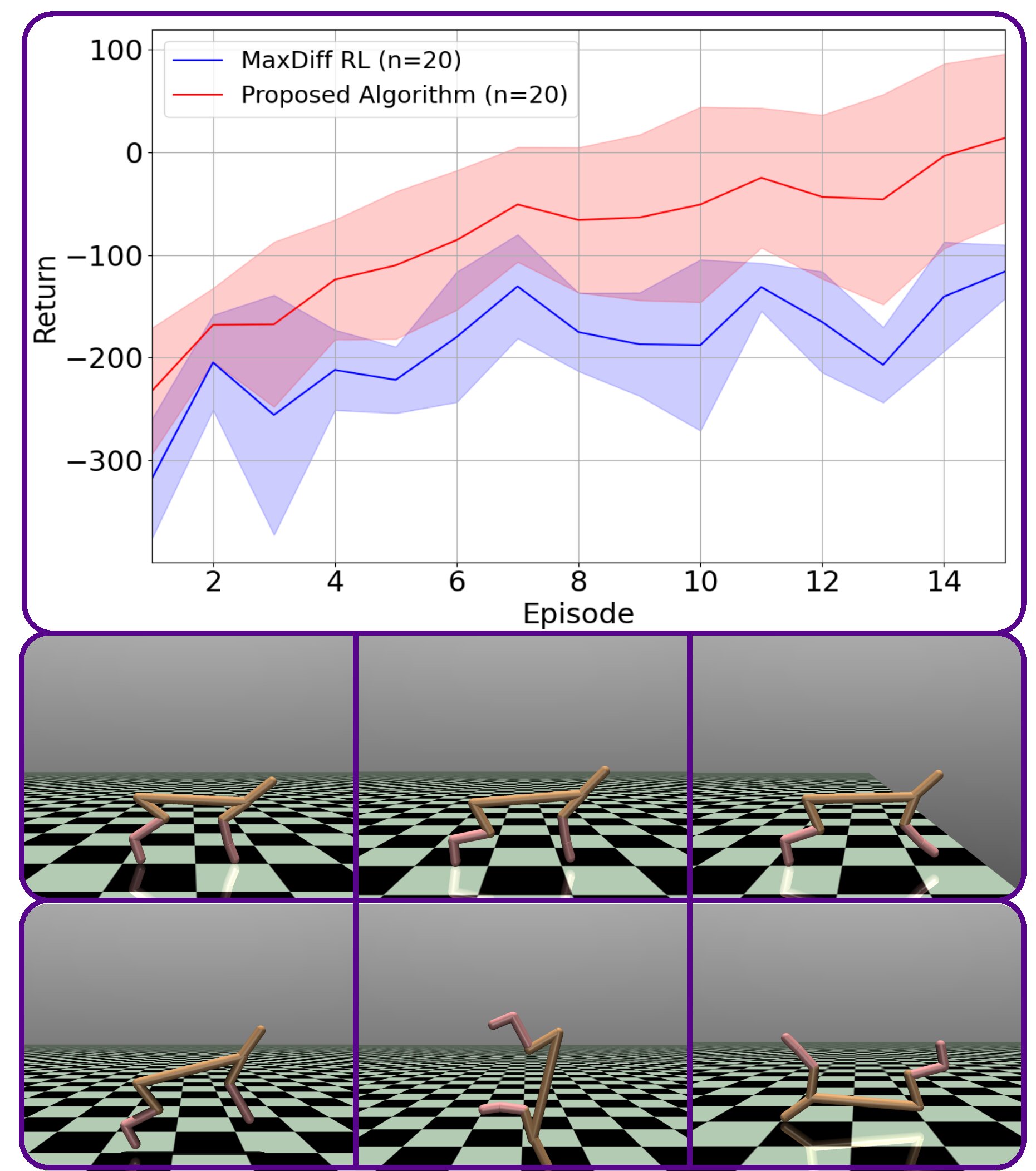}
 \caption{HalfCheetah results. Top: training performance of the proposed algorithm vs.\ MaxDiff RL on HalfCheetah-v5 (mean $\pm$1\,STD over $n=20$ runs). Middle: rollout frames from the proposed algorithm showing a stable stride reaching the target. Bottom: rollout frames from MaxDiff RL illustrating an unstable gait leading to failure.}
  \label{fig:hc-overview}\vspace{-0.5cm}
\end{figure}

\textbf{Results.}
Figure~\ref{fig:hc-overview}-top compares learning curves for DR-FREE and a
MaxDiff baseline. DR-FREE shows steady improvement in return with lower
variance early in training. Snapshots in
Fig.~\ref{fig:hc-overview}-middle and bottom panels illustrate qualitatively smoother, more
stable gaits relative to the higher-variance baseline. For the HalfCheetah benchmark, after training we performed 20 evaluation rollouts: our proposed method successfully reached the goal 18 times, while the MaxDiff baseline succeeded only 6 times.

\subsection{Franka Obstacle Task (simulation)}
\label{subsec:franka-sim}

\textbf{Environment.}
A tabletop manipulation scene with a vertical obstacle between the start
pose and the goal (grasp-and-place of a green block).



\begin{figure}[t]
	\centering
	\def\svgwidth{\linewidth}
	{\fontsize{8}{8}
		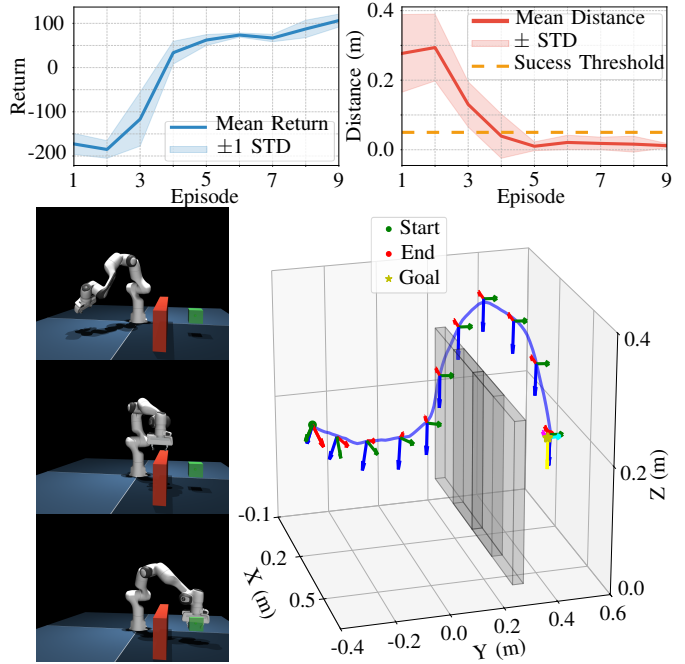}
\caption{Franka obstacle-avoidance results with the proposed method. Top: learning curves showing episode return and minimum distance to goal (success threshold 5\,cm, mean $\pm$1\,STD over $n=20$ runs). Bottom: i) a rollout demonstrating collision-free manipulation around the obstacle. ii) trajectory plot with goal orientation highlighted in cyan, magenta, and yellow, illustrating efficient training and task execution.}
  \label{fig:franka-obstacle-overview}
\end{figure}

\textbf{Results.}
Figure~\ref{fig:franka-obstacle-overview}-top shows returns increasing as
the minimum distance-to-goal falls below the 5\,cm threshold. The
snapshots in Fig.~\ref{fig:franka-obstacle-overview}b depict a learned,
collision-free path that adapts around the obstacle. The controller’s
cautious lateral adjustment emerges from the ambiguity cost, which is larger near contact-uncertain poses and thus
biases the Gibbs kernel toward safer actions.

\subsection{Franka Research~3 (real robot)}
\label{subsec:franka-real}

\textbf{Deployment.} 
We deploy the proposed algorithm on the real Franka Research~3 arm for a pick-and-place task in a cluttered tabletop scene, directly transferring the learned dynamics and executing the greedy controller without any task-specific tuning. The robot must reach a first goal position to grasp a green cube and then transport it to a second goal location. Two representative executions are shown in Fig.~\ref{fig:franka_pick_place}:
\begin{enumerate}
  \item \textbf{No-obstacle:} After grasping the cube at the first goal, the policy selects the optimal straight path across the table and successfully places the object at the second goal; see Fig.~\ref{fig:franka_pick_place}-top.
  \item \textbf{Obstacle-present:} When an obstacle is encountered en route, the policy autonomously plans a collision-free strategy by elevating the gripper above the obstruction before advancing toward and placing the object at the second goal; see Fig.~\ref{fig:franka_pick_place}-bottom.
\end{enumerate}
\begin{figure}
    \centering
    \includegraphics[width=0.48\textwidth]{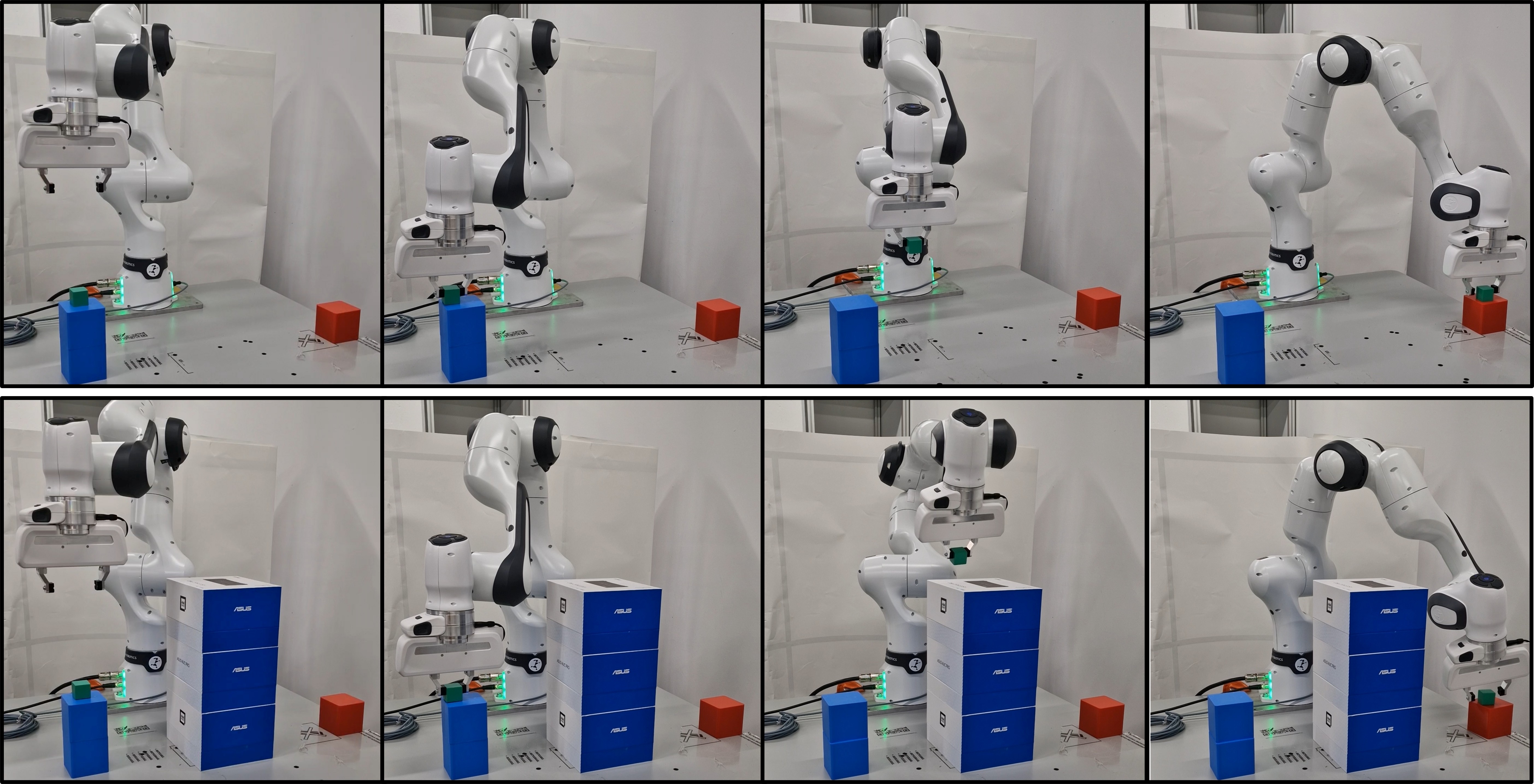}
    \caption{Deployment on the Franka Research~3 for cluttered tabletop pick-and-place. Top: grasp at the first goal and place at the second. Bottom: with an obstacle, the controller lifts to avoid collision and completes the placement.}
    \label{fig:franka_pick_place}
\end{figure}

\textbf{Interpretation.}  
 Qualitatively distinct, yet collision-free behaviors in these settings arise directly from the optimal policy in \eqref{eq:gibbs-policy}. With a non-informative action prior, adaptation is governed by the Gibbs exponent: the goal-centered state prior pulls the system toward the target state \(\bv{x}^\star\), while the obstacle cost increases in regions near obstacles. This modulation naturally leads the policy to take the direct path when unobstructed, or to switch to a lift-over maneuver when obstacles are present, ensuring safe task execution.

In Table~\ref{tab:eta_comparison}, we report a sensitivity analysis of the scaling coefficient $\varrho$ for obstacle-aware pick-and-place. As $\varrho$ increases, the learned policy becomes increasingly goal-seeking and less responsive to obstacles, leading to more frequent collisions.
\begin{table}[!t]
\centering
\begin{tabular}{
    p{2cm}  
    p{0.7cm} 
    p{0.4cm} 
    p{0.4cm} 
    p{0.4cm} 
    p{0.4cm} 
    p{0.4cm} 
    p{0.4cm} 
}
\toprule
Metric & $\varrho=0$ & $0.5$ & $1$ & $5$ & $100$ & $1000$ & $2000$ \\
\midrule
Normalized Avg. Episode Cost ($\downarrow$)
  & $1.87$              & $1.01$              & $\mathbf{1.00}$
  & $1.24$              & $9.50$              & $56.8$              & $148$ \\
Standard Deviation
  & $0.076$             & $0.014$             & $0.055$
  & $0.099$             & $7.96$              & $11.4$              & $28.3$ \\
Success Rate ($\uparrow$)
  & $100\%$ & $100\%$ & $100\%$ & $80\%$ & $60\%$ & $20\%$ & $0\%$ \\
\bottomrule
\end{tabular}
\caption{Performance across $\varrho$ over 5 trials. Episode cost is normalized to the best case ($\varrho{=}1$). A trial succeeds if the robot reaches the goal while staying $\geq 0.07$,m from all obstacle centers throughout the episode.}\vspace{-0.5cm}
\label{tab:eta_comparison}
\end{table}
\section{Conclusion}
Building on the distributionally robust free energy principle, we introduced a modification to the maximum diffusion framework that simultaneously enables policy learning for continuous robotics control tasks and provides a priori robustness guarantees. The resulting computational model enables robust exploration by combining maximum-diffusion state generation with per–state-action KL ambiguity sets, yielding a tractable min–max objective and a robust optimal policy. The same min–max structure extends to bounded stage-cost perturbations via an augmented-state formulation. Both numerical and hardware experiments illustrate that our policy computation model convincingly outperforms MaxDiff baselines in the HalfCheetah environment and enables zero-shot deployments on real hardware.

\bibliographystyle{IEEEtran}
\bibliography{references} 

\end{document}

%% file: my_commands.tex
\newcommand{\uSt}{\mathbf{u}^*_k}
\newcommand{\InScheme}{\{\mathbf{x}_k,c_k,\mathbf{x}_{k-1},\mathbf{u}_k\}}

\newcommand{\sB}{\mathcal{B}}

\newcommand{\sD}{\mathcal{D}}

\newcommand{\sF}{\mathcal{F}}

\newcommand{\E}{\mathbb{E}}

\newcommand{\bv}[1]{\mathbf{#1}}

\newcommand{\Vx}[2]{V_{\alpha}\left(\bv{x}_{k-1},\bv{u}_k\right)}
\newcommand{\Wx}[2]{W_{\alpha}\left(\bv{x}_{k-1},\bv{u}_k\right)}
\newcommand{\Vxtilde}[2]{\tilde{V}_{\alpha}\left(\bv{x}_{k-1},\bv{u}_k\right)}
\newcommand{\Mx}[2]{M\left(\bv{x}_{k-1},\bv{u}_k\right)}

\newcommand{\plant}[2]{p_{{#1}} \left(\bv{x}_{{#1}}\mid \bv{x}_{{#2}}, \bv{u}_{{#1}} \right)}

\newcommand{\refplant}[2]{q_{{#1}} \left(\bv{x}_{{#1}}\mid \bv{x}_{{#2}}, \bv{u}_{{#1}} \right)}

\newcommand{\refpolicy}[2]{q_{{#1}}\left(\bv{u}_{{#1}}\mid \bv{x}_{{#2}} \right)}

\newcommand{\KL}{\text{KL}}

\newcommand{\DKL}[2]{{D}_{\KL}\left(#1\mid \mid #2 \right)}

\newtheorem{theorem}{Theorem}[section]

\newtheorem{definition}{Definition}[section]

\newtheorem{remark}{Remark}[section]

\usepackage{dsfont}

\renewcommand{\thetable}{S-\arabic{table}}

%% file: figures/scheme.pdf_tex
\begingroup%
  \makeatletter%
  \providecommand\color[2][]{%
    \errmessage{(Inkscape) Color is used for the text in Inkscape, but the package 'color.sty' is not loaded}%
    \renewcommand\color[2][]{}%
  }%
  \providecommand\transparent[1]{%
    \errmessage{(Inkscape) Transparency is used (non-zero) for the text in Inkscape, but the package 'transparent.sty' is not loaded}%
    \renewcommand\transparent[1]{}%
  }%
  \providecommand\rotatebox[2]{#2}%
  \newcommand*\fsize{\dimexpr\f@size pt\relax}%
  \newcommand*\lineheight[1]{\fontsize{\fsize}{#1\fsize}\selectfont}%
  \ifx\svgwidth\undefined%
    \setlength{\unitlength}{291.55759796bp}%
    \ifx\svgscale\undefined%
      \relax%
    \else%
      \setlength{\unitlength}{\unitlength * \real{\svgscale}}%
    \fi%
  \else%
    \setlength{\unitlength}{\svgwidth}%
  \fi%
  \global\let\svgwidth\undefined%
  \global\let\svgscale\undefined%
  \makeatother%
  \begin{picture}(1,1.36699083)%
    \lineheight{1}%
    \setlength\tabcolsep{0pt}%
    \put(0,0){\includegraphics[width=\unitlength,page=1]{scheme.pdf}}%
    \put(0.37656607,1.2007341){\color[rgb]{1,1,1}\makebox(0,0)[t]{\lineheight{1.25}\smash{\begin{tabular}[t]{c}Overall Optimization Objective  \end{tabular}}}}%
    \put(0.50752264,0.97309506){\color[rgb]{1,1,1}\makebox(0,0)[t]{\lineheight{1.25}\smash{\begin{tabular}[t]{c}\textbf{Replay Buffer}\end{tabular}}}}%
    \put(0,0){\includegraphics[width=\unitlength,page=2]{scheme.pdf}}%
    \put(0.84770572,0.9786386){\color[rgb]{1,1,1}\makebox(0,0)[t]{\lineheight{1.25}\smash{\begin{tabular}[t]{c}\textbf{Dynamics \&}\end{tabular}}}}%
    \put(0.84647983,0.93113852){\color[rgb]{1,1,1}\makebox(0,0)[t]{\lineheight{1.25}\smash{\begin{tabular}[t]{c}\textbf{Cost Model}\end{tabular}}}}%
    \put(0.84180858,1.01448914){\color[rgb]{1,1,1}\makebox(0,0)[t]{\lineheight{1.25}\smash{\begin{tabular}[t]{c}\textbf{Training}\end{tabular}}}}%
    \put(0.61741445,0.70325598){\color[rgb]{1,1,1}\makebox(0,0)[t]{\lineheight{1.25}\smash{\begin{tabular}[t]{c}\textbf{Compute $p_{max}$}\end{tabular}}}}%
    \put(0.60952081,0.53299077){\color[rgb]{1,1,1}\makebox(0,0)[t]{\lineheight{1.25}\smash{\begin{tabular}[t]{c}\textbf{Solve Inner}\end{tabular}}}}%
    \put(0.61275262,0.49056635){\color[rgb]{1,1,1}\makebox(0,0)[t]{\lineheight{1.25}\smash{\begin{tabular}[t]{c}\textbf{Maximization}\end{tabular}}}}%
    \put(0.6109067,0.37875983){\color[rgb]{1,1,1}\makebox(0,0)[t]{\lineheight{1.25}\smash{\begin{tabular}[t]{c}\textbf{Solve Outer}\end{tabular}}}}%
    \put(0.61185125,0.33691005){\color[rgb]{1,1,1}\makebox(0,0)[t]{\lineheight{1.25}\smash{\begin{tabular}[t]{c}\textbf{Minimization}\end{tabular}}}}%
    \put(0.60968598,0.20722649){\color[rgb]{1,1,1}\makebox(0,0)[t]{\lineheight{1.25}\smash{\begin{tabular}[t]{c}\textbf{Twisted Kernel Policy}\end{tabular}}}}%
    \put(0.61138196,0.04233239){\color[rgb]{1,1,1}\makebox(0,0)[t]{\lineheight{1.25}\smash{\begin{tabular}[t]{c}\textbf{Sampled Action $\uSt$}\end{tabular}}}}%
    \put(0.61207451,0.78519299){\color[rgb]{0,0,0}\makebox(0,0)[t]{\lineheight{1.25}\smash{\begin{tabular}[t]{c}\textbf{Agent's Decision Cycle}\end{tabular}}}}%
    \put(0.61042061,0.58102197){\color[rgb]{0,0,0}\makebox(0,0)[t]{\lineheight{1.25}\smash{\begin{tabular}[t]{c}\textbf{Solve Min-Max Optimization}\end{tabular}}}}%
    \put(0.77514576,0.43904066){\color[rgb]{0,0,0}\makebox(0,0)[t]{\lineheight{1.25}\smash{\begin{tabular}[t]{c}Cost of Ambiguity\end{tabular}}}}%
    \put(0,0){\includegraphics[width=\unitlength,page=3]{scheme.pdf}}%
    \put(0.17556949,0.45720487){\color[rgb]{0,0,0}\makebox(0,0)[t]{\lineheight{1.25}\smash{\begin{tabular}[t]{c}$\mathbf{x}_{k-1}$\end{tabular}}}}%
    \put(0.1498167,0.97441486){\color[rgb]{0,0,0}\makebox(0,0)[t]{\lineheight{1.25}\smash{\begin{tabular}[t]{c}$\InScheme$\end{tabular}}}}%
  \end{picture}%
\endgroup%

%% file: figures/frankaSim.pdf_tex
\begingroup%
  \makeatletter%
  \providecommand\color[2][]{%
    \errmessage{(Inkscape) Color is used for the text in Inkscape, but the package 'color.sty' is not loaded}%
    \renewcommand\color[2][]{}%
  }%
  \providecommand\transparent[1]{%
    \errmessage{(Inkscape) Transparency is used (non-zero) for the text in Inkscape, but the package 'transparent.sty' is not loaded}%
    \renewcommand\transparent[1]{}%
  }%
  \providecommand\rotatebox[2]{#2}%
  \newcommand*\fsize{\dimexpr\f@size pt\relax}%
  \newcommand*\lineheight[1]{\fontsize{\fsize}{#1\fsize}\selectfont}%
  \ifx\svgwidth\undefined%
    \setlength{\unitlength}{296.93387009bp}%
    \ifx\svgscale\undefined%
      \relax%
    \else%
      \setlength{\unitlength}{\unitlength * \real{\svgscale}}%
    \fi%
  \else%
    \setlength{\unitlength}{\svgwidth}%
  \fi%
  \global\let\svgwidth\undefined%
  \global\let\svgscale\undefined%
  \makeatother%
  \begin{picture}(1,1.00798011)%
    \lineheight{1}%
    \setlength\tabcolsep{0pt}%
    \put(0,0){\includegraphics[width=\unitlength,page=1]{frankaSim.pdf}}%
    \put(0.04498227,0.97291852){\color[rgb]{0,0,0}\makebox(0,0)[t]{\lineheight{1.25}\smash{\begin{tabular}[t]{c}100\end{tabular}}}}%
    \put(0.05194809,0.90473233){\color[rgb]{0,0,0}\makebox(0,0)[t]{\lineheight{1.25}\smash{\begin{tabular}[t]{c}0\end{tabular}}}}%
    \put(0,0){\includegraphics[width=\unitlength,page=2]{frankaSim.pdf}}%
    \put(0.50583619,0.01011811){\color[rgb]{0,0,0}\makebox(0,0)[t]{\lineheight{1.25}\smash{\begin{tabular}[t]{c}-0.4\end{tabular}}}}%
    \put(0,0){\includegraphics[width=\unitlength,page=3]{frankaSim.pdf}}%
    \put(0.04303476,0.83607805){\color[rgb]{0,0,0}\makebox(0,0)[t]{\lineheight{1.25}\smash{\begin{tabular}[t]{c}-100\end{tabular}}}}%
    \put(0.04303476,0.77348017){\color[rgb]{0,0,0}\makebox(0,0)[t]{\lineheight{1.25}\smash{\begin{tabular}[t]{c}-200\end{tabular}}}}%
    \put(0.99715925,0.73080872){\color[rgb]{0,0,0}\makebox(0,0)[t]{\lineheight{1.25}\smash{\begin{tabular}[t]{c}9\end{tabular}}}}%
    \put(0.49215937,0.73080872){\color[rgb]{0,0,0}\makebox(0,0)[t]{\lineheight{1.25}\smash{\begin{tabular}[t]{c}9\end{tabular}}}}%
    \put(0.89522423,0.73082154){\color[rgb]{0,0,0}\makebox(0,0)[t]{\lineheight{1.25}\smash{\begin{tabular}[t]{c}7\end{tabular}}}}%
    \put(0.39026116,0.73082154){\color[rgb]{0,0,0}\makebox(0,0)[t]{\lineheight{1.25}\smash{\begin{tabular}[t]{c}7\end{tabular}}}}%
    \put(0.28830603,0.73134292){\color[rgb]{0,0,0}\makebox(0,0)[t]{\lineheight{1.25}\smash{\begin{tabular}[t]{c}5\end{tabular}}}}%
    \put(0.79330588,0.73134292){\color[rgb]{0,0,0}\makebox(0,0)[t]{\lineheight{1.25}\smash{\begin{tabular}[t]{c}5\end{tabular}}}}%
    \put(0.6915957,0.73079576){\color[rgb]{0,0,0}\makebox(0,0)[t]{\lineheight{1.25}\smash{\begin{tabular}[t]{c}3\end{tabular}}}}%
    \put(0.18659583,0.73079576){\color[rgb]{0,0,0}\makebox(0,0)[t]{\lineheight{1.25}\smash{\begin{tabular}[t]{c}3\end{tabular}}}}%
    \put(0.0836049,0.73080938){\color[rgb]{0,0,0}\makebox(0,0)[t]{\lineheight{1.25}\smash{\begin{tabular}[t]{c}1\end{tabular}}}}%
    \put(0.58945568,0.73072478){\color[rgb]{0,0,0}\makebox(0,0)[t]{\lineheight{1.25}\smash{\begin{tabular}[t]{c}1\end{tabular}}}}%
    \put(0.55675312,0.77807919){\color[rgb]{0,0,0}\makebox(0,0)[t]{\lineheight{1.25}\smash{\begin{tabular}[t]{c}0.0\end{tabular}}}}%
    \put(0.94239418,0.10154943){\color[rgb]{0,0,0}\makebox(0,0)[t]{\lineheight{1.25}\smash{\begin{tabular}[t]{c}0.0\end{tabular}}}}%
    \put(0.55675636,0.88514629){\color[rgb]{0,0,0}\makebox(0,0)[t]{\lineheight{1.25}\smash{\begin{tabular}[t]{c}0.2\end{tabular}}}}%
    \put(0.39222809,0.14771436){\color[rgb]{0,0,0}\makebox(0,0)[t]{\lineheight{1.25}\smash{\begin{tabular}[t]{c}0.2\end{tabular}}}}%
    \put(0.55674664,0.99214841){\color[rgb]{0,0,0}\makebox(0,0)[t]{\lineheight{1.25}\smash{\begin{tabular}[t]{c}0.4\end{tabular}}}}%
    \put(0.52565486,0.88631329){\color[rgb]{0,0,0}\rotatebox{90}{\makebox(0,0)[t]{\lineheight{1.25}\smash{\begin{tabular}[t]{c}Distance (m)\end{tabular}}}}}%
    \put(0.00887848,0.88517917){\color[rgb]{0,0,0}\rotatebox{90}{\makebox(0,0)[t]{\lineheight{1.25}\smash{\begin{tabular}[t]{c}Return\end{tabular}}}}}%
    \put(0.78818622,0.70607466){\color[rgb]{0,0,0}\makebox(0,0)[t]{\lineheight{1.25}\smash{\begin{tabular}[t]{c}Episode\end{tabular}}}}%
    \put(0.87800582,0.9120289){\color[rgb]{0,0,0}\makebox(0,0)[t]{\lineheight{1.25}\smash{\begin{tabular}[t]{c}Sucess Threshold\end{tabular}}}}%
    \put(0.28594728,0.70641457){\color[rgb]{0,0,0}\makebox(0,0)[t]{\lineheight{1.25}\smash{\begin{tabular}[t]{c}Episode\end{tabular}}}}%
    \put(0.86388325,0.98127147){\color[rgb]{0,0,0}\makebox(0,0)[t]{\lineheight{1.25}\smash{\begin{tabular}[t]{c}Mean Distance\end{tabular}}}}%
    \put(0.81049684,0.9465218){\color[rgb]{0,0,0}\makebox(0,0)[t]{\lineheight{1.25}\smash{\begin{tabular}[t]{c}$\pm$ STD\end{tabular}}}}%
    \put(0.75173942,0.04287934){\color[rgb]{0,0,0}\makebox(0,0)[t]{\lineheight{1.25}\smash{\begin{tabular}[t]{c}0.2\end{tabular}}}}%
    \put(0.94615383,0.27896502){\color[rgb]{0,0,0}\makebox(0,0)[t]{\lineheight{1.25}\smash{\begin{tabular}[t]{c}0.2\end{tabular}}}}%
    \put(0.95206963,0.48613659){\color[rgb]{0,0,0}\makebox(0,0)[t]{\lineheight{1.25}\smash{\begin{tabular}[t]{c}0.4\end{tabular}}}}%
    \put(0.83557373,0.05334714){\color[rgb]{0,0,0}\makebox(0,0)[t]{\lineheight{1.25}\smash{\begin{tabular}[t]{c}0.4\end{tabular}}}}%
    \put(0.6696456,0.03276041){\color[rgb]{0,0,0}\makebox(0,0)[t]{\lineheight{1.25}\smash{\begin{tabular}[t]{c}0.0\end{tabular}}}}%
    \put(0.36305772,0.21576798){\color[rgb]{0,0,0}\makebox(0,0)[t]{\lineheight{1.25}\smash{\begin{tabular}[t]{c}-0.1\end{tabular}}}}%
    \put(0.59311406,0.01975117){\color[rgb]{0,0,0}\makebox(0,0)[t]{\lineheight{1.25}\smash{\begin{tabular}[t]{c}-0.2\end{tabular}}}}%
    \put(0.91457483,0.06252237){\color[rgb]{0,0,0}\makebox(0,0)[t]{\lineheight{1.25}\smash{\begin{tabular}[t]{c}0.6\end{tabular}}}}%
    \put(0.43651644,0.08280006){\color[rgb]{0,0,0}\makebox(0,0)[t]{\lineheight{1.25}\smash{\begin{tabular}[t]{c}0.5\end{tabular}}}}%
    \put(0.36243156,0.09948388){\color[rgb]{0,0,0}\rotatebox{-66.195776}{\makebox(0,0)[t]{\lineheight{1.25}\smash{\begin{tabular}[t]{c}X (m)\end{tabular}}}}}%
    \put(0.74093175,0.00789164){\color[rgb]{0,0,0}\rotatebox{8.2031348}{\makebox(0,0)[t]{\lineheight{1.25}\smash{\begin{tabular}[t]{c}Y (m)\end{tabular}}}}}%
    \put(0.9935184,0.28712211){\color[rgb]{0,0,0}\rotatebox{90}{\makebox(0,0)[t]{\lineheight{1.25}\smash{\begin{tabular}[t]{c}Z (m)\end{tabular}}}}}%
    \put(0.61702955,0.65748438){\color[rgb]{0,0,0}\makebox(0,0)[t]{\lineheight{1.25}\smash{\begin{tabular}[t]{c}Start\end{tabular}}}}%
    \put(0.61435286,0.61930517){\color[rgb]{0,0,0}\makebox(0,0)[t]{\lineheight{1.25}\smash{\begin{tabular}[t]{c}End\end{tabular}}}}%
    \put(0.61680318,0.58099792){\color[rgb]{0,0,0}\makebox(0,0)[t]{\lineheight{1.25}\smash{\begin{tabular}[t]{c}Goal\end{tabular}}}}%
    \put(0.3897837,0.81838808){\color[rgb]{0,0,0}\makebox(0,0)[t]{\lineheight{1.25}\smash{\begin{tabular}[t]{c}Mean Return\end{tabular}}}}%
    \put(0.36148252,0.78341574){\color[rgb]{0,0,0}\makebox(0,0)[t]{\lineheight{1.25}\smash{\begin{tabular}[t]{c}$\pm1$ STD\end{tabular}}}}%
    \put(0,0){\includegraphics[width=\unitlength,page=4]{frankaSim.pdf}}%
  \end{picture}%
\endgroup%